%% file: main.tex
\documentclass[11pt]{article}
\input{def}

\title{Scalable Algorithms for Learning High-Dimensional Linear Mixed Models}

\author{
Zilong Tan\\
Department of Computer Science\\
Duke University \\
{\tt ztan@cs.duke.edu}
\and
Kimberly Roche\\
Program in Computational Biology \& Bioinformatics \\
Duke University \\
{\tt kimberly.roche@duke.edu}
\and
Xiang Zhou\\
Department of Biostatistics \&\\
Center for Statistical Genetics\\
University of Michigan\\
{\tt xzhousph@umich.edu}
\and
Sayan Mukherjee\\
Departments of Statistical Science,\\
Mathematics, Computer Science, \\
Biostatistics \& Bioinformatics \\
Duke University \\
{\tt sayan@stat.duke.edu}}

\begin{document}

\maketitle

\begin{abstract}
Linear mixed models (LMMs) are used extensively to model 
dependecies of observations in linear regression and are used extensively in many application areas. Parameter estimation for LMMs can be computationally prohibitive on big data. State-of-the-art learning algorithms require computational complexity which depends at least linearly on the dimension $p$ of the covariates, and often use heuristics that do not offer theoretical guarantees. We present scalable algorithms for learning high-dimensional LMMs with sublinear computational complexity dependence on $p$. Key to our approach are novel dual estimators which use only kernel functions of the data, and fast computational techniques based on the subsampled randomized Hadamard transform. We provide theoretical guarantees for our learning algorithms, demonstrating the robustness of parameter estimation. Finally, we complement the theory with experiments on large synthetic and real data.
\end{abstract}

\section{Introduction} 
Linear mixed models (LMMs) are widely used in many real world applications ranging from longitudinal data analysis \cite{Laird82,Demidenko13} and genome wide association studies \cite{Kang08,Lippert11,Zhou14,Zhou17} to recommender systems \cite{Zhang08,Zhang16}. LMMs provide a flexible framework for modeling a wide range of data types, including clustered, longitudinal, and spatial data. Parameter estimation for LMMs is computationally prohibitive for big data, both for large sample size $n$ \cite{Zhou14,Darnell17,Perry17} and for high-dimensional covariates $p$ \cite{Schelldorfer11}. The main computational bottlenecks for parameter estimation arise from the non-convexity of the optimization problem \cite{Kang08,Perry17} as well as the computational cost of matrix inversions \cite{Laird87,Lindstrom88,Bates15,Zhou17}. State-of-the-art methods for parameter estimation in LMMs require computational complexity that depends at least linearly on $p$: (i) $O\left(n k p\right)$ for the setting $n > p$ with a rank $k$ covariance matrix \cite{Zhou17,Darnell17}; and (ii) $O\left(n^2 p\right)$ per iteration for $p \gg n$ \cite{Schelldorfer11,Schelldorfer14,Jakubik15}. In this paper, we present scalable algorithms with sublinear computational complexity in $p$, making the proposed approach useful for high-dimensional LMMs. In addition, we provide a theoretical analysis for our approach that states provable error guarantees between the estimated and ground-truth parameters.

Two sets of parameters are estimated in LMMs, the fixed-effects coefficients and the variances for the unobservable random effects and noise. The random-effects variance is generally assumed to have a certain structure, such as a block-diagonal matrix \cite{Laird82,Demidenko13}. To estimate both sets of parameters, an expectation maximization (EM) algorithm is typically used \cite{Laird87,Bates15} to handle the latent random-effects variables. The M-step in the EM algorithm incurs high computational costs due to matrix inversions. Newton-Raphson has been used to reduce the number of iterations required for parameter estimates to converge \cite{Lindstrom88}; however, each iteration is still costly due to matrix inversions. A recent research focus is to avoid matrix inversions at each iteration. For instance, when $n > p$ a spectral algorithm is available \cite{Patterson71,Kang08,Lippert11}. The state-of-the-art algorithm \cite{Darnell17} further improved the computational complexity of the spectral algorithm using randomized singular value decomposition \cite{Darnell17}.

While approximate learning algorithms \cite{Zhou17,Darnell17} are efficient, few provide provable guarantees in terms of estimation accuracy. Recently, a 
non-iterative algorithm with provable guarantees for estimates was proposed in \cite{Perry17}, which runs in $O\left(n\left(p + d\right)^4\right)$ time for $d$ random effects. Inference with guarantees for high-dimensional LMMs, i.e., $p \gg n$, typically incurs greater computational complexity due to the regularization required to address high-dimensional data \cite{Schelldorfer11,Schelldorfer14}. In the high-dimensional setting, most algorithms perform block coordinate descent with an $O\left(n^2 p\right)$ per-iteration cost \cite{Schelldorfer11,Schelldorfer14}. In this paper, we show that efficiency and provable guarantees can be achieved simultaneously for learning high-dimensional LMMs.

There are two key ideas we use in our efficient algorithms. The first idea is using an approximate estimator that relies on an $n\times n$ kernel matrix (\cref{sec:estimators}) which can be computed efficiently using the subsampled randomized Hadamard transform (SRHT) \cite{Tropp11}. This reduces the linear complexity dependence on $p$. Unlike some other approximation algorithms \cite{Lu13}, the proposed estimator also has the advantage of recovering the fixed-effects coefficients for all $p$ dimensions as opposed to the reduced dimensions. This allows us to provide effect sizes for all the original covariates, a requirement in many applications. The second idea is the introduction of {\em approximate variance components} (AVCs) to replace variance components when estimating the fixed-effects coefficients. These AVCs have a closed-form expression and can be computed efficiently.
  
We apply our novel approach to LMMs with both general covariance matrices as well as block-diagonal covariance matrices for the random effects. The former can be viewed as a special case of the latter with a single block, and has been adopted in genome-wide association studies \cite{Kang08,Lippert11,Zhou17}. LMMs with a block-diagonal covariance structure have been widely used for modeling repeated measures data \cite{Laird82} as well as for modeling batch effects \cite{Johnson07}. We propose a non-iterative algorithm for the general covariance setting and a fast EM variant for the block-diagonal setting.

\paragraph{Contribution}
Our main contribution is providing a class of approximation algorithms for parameter inference in high-dimensional LMMs with provable guarantees. In \cref{tbl:methods}, we state the computational complexity for several standard and state-of-the-art parameter inference algorithms. In the table and throughout this paper, $n$ is the sample size, $p$ is the number of covariates, $k$ is the rank of the covariance matrix, $s$ is the number of subsamples, and $\epsilon$ is the approximation error. Our method is the only one that is sublinear in $p$, and can be a $n/\log p$ magnitude faster than the others (discussed in \cref{sec:fixedLMM}). In addition to theoretical advantages, we demonstrate the empirical accuracy and speed of our method on both synthetic and real data in \cref{sec:experiments}. 

\begin{table}[tb]
\caption{Computational complexity for parameter inference. \textdagger \, denotes that the estimator has provable guarantees.}
\label{tbl:methods}
\begin{center}
\begin{threeparttable}
\begin{small}
\begin{sc}
\begin{tabular}{c|c}
\hline
REML \cite{Kang08,Lippert11} & $O\left(n^2 p\right)$\\ \hline
\tnote{\textdagger} Method-of-moments \cite{Perry17} & $O\left(n \left(p+q\right)^4\right)$\\ \hline
Subsampling \cite{Zhou17} & $O\left(p s^2\right)$\\ \hline
ARSVD \cite{Darnell17} & $O\left(p n k\right)$\\ \hline
\tnote{\textdagger} This work & $O\left( \frac{ n^2 \left(k + \log p\right) \log k }{\epsilon^2}\right)$\\ \hline
\end{tabular}
\end{sc}
\end{small}
\end{threeparttable}
\end{center}
\end{table}

\paragraph{Notation}
We denote the maximum and minimum eigenvalues of a matrix $\bm{A}$ by $\lambda_{\text{max}}\left(\bm{A}\right)$ and $\lambda_{\text{min}}\left(\bm{A}\right)$, respectively. Similarly, we denote the $j$-th, maximum, and minimum singular values respectively by $\sigma_j\left(\bm{A}\right)$, $\sigma_{\text{max}}\left(\bm{A}\right)$, and $\sigma_{\text{min}}\left(\bm{A}\right)$. $\bm{A}^\dagger$ represents the Moore–Penrose pseudoinverse of $\bm{A}$, and $\kappa\left(\bm{A}\right)$ denotes the condition number of $\bm{A}$. The superscripted notation $\bm{y}^{\left(i\right)}$ refers to the copy of $\bm{y}$ for group $i$. We state the spectral norm of a matrix as $\left\|\cdot\right\|_2$, the Frobenius norm as $\left\|\cdot\right\|_F$, and the Ky Fan $k$-norm (the sum of the $k$ largest singular values) as $\vertiii{\cdot}_k$.

\paragraph{Organization} Section \ref{sec:rlmm} provides background on standard LMMs. In section \ref{sec:estimators}, we formulate $L_2$-regularized LMMs and present approximate estimators based on a kernel matrix. Section \ref{sec:kernel-comp} states a fast procedure to compute the approximate estimators. In section \ref{sec:guarantees}, we provide theoretical guarantees for our estimators. Section \ref{sec:experiments} reports empirical evidence of the speed and accuracy of our methods, and section \ref{sec:concl} concludes this paper.

\section{Linear Mixed Models}
\label{sec:rlmm}
Consider a regression problem with $n$ observations, where $\bm{y} \in \mathbb{R}^{n}$ denotes the response vector and $\bm{X} \in \mathbb{R}^{n \times p}$ represents the covariates matrix with $p$ covariates. The standard LMM is given by
\begin{gather}
\begin{split}
\bm{y} = \bm{X} \bm{\beta} + \bm{Z} \bm{\gamma} + c\bm{1} + \bm{e} \qquad \text{with} \qquad
\begin{bmatrix}
\bm{\gamma}\\
\bm{e}
\end{bmatrix}
\sim \text{MVN}\left(\bm{0},
\begin{bmatrix}
\bm{\Lambda} & \bm{0}\\
\bm{0} & \sigma^2\bm{I}
\end{bmatrix}
\right),
\end{split}
\label{eq:lmm}
\end{gather}
where $\bm{\beta} \in \mathbb{R}^p$ is the fixed-effects coefficient vector, $\bm{Z} \in \mathbb{R}^{n \times q}$ is a full-rank random-effects design matrix, $\bm{\gamma} \in \mathbb{R}^{q}$ is the random-effects coefficient vector, c is the intercept, and $\bm{e} \in \mathbb{R}^n$ is the noise vector. The parameters to be estimated are the fixed-effects coefficients $\bm{\beta}$, and variance components $\bm{\Lambda}$ and $\sigma^2$. 

In general, the variables $\bm{X}$, $\bm{y}$, $\bm{\gamma}$, and $\bm{e}$ in \eqref{eq:lmm} correspond to observations from $m$ classes, and are grouped by the following structure \cite{Laird82}:
\begin{align*}
\bm{X} = \begin{bmatrix}
\bm{X}^{\left(1\right)}\\
\bm{X}^{\left(2\right)}\\
\vdots\\
\bm{X}^{\left(m\right)}
\end{bmatrix}, \qquad
\bm{y} = \begin{bmatrix}
\bm{y}^{\left(1\right)}\\
\bm{y}^{\left(2\right)}\\
\vdots\\
\bm{y}^{\left(m\right)}
\end{bmatrix}, \qquad
\bm{\gamma} = \begin{bmatrix}
\bm{\gamma}^{\left(1\right)}\\
\bm{\gamma}^{\left(2\right)}\\
\vdots\\
\bm{\gamma}^{\left(m\right)}
\end{bmatrix}, \qquad
\bm{e} = \begin{bmatrix}
\bm{e}^{\left(1\right)}\\
\bm{e}^{\left(2\right)}\\
\vdots\\
\bm{e}^{\left(m\right)}
\end{bmatrix},
\end{align*}
where $\cdot^{\left(i\right)}$ denote the variables specific to group $i$, whose dimensions are $\bm{X}^{\left(i\right)} \in \mathbb{R}^{n_i \times p}$, $\bm{\gamma}^{\left(i\right)} \in \mathbb{R}^d$, and $\bm{y}^{\left(i\right)},\bm{e}^{\left(i\right)} \in \mathbb{R}^{n_i}$, $\sum_{i=1}^m n_i = n$. The LMM assumes that the $\bm{\gamma}^{\left(i\right)}$ corresponding to distinct classes are independent. In particular, the random-effects design matrix $\bm{Z}$ and the random-effects covariance are block-diagonal 
\begin{align*}
\bm{Z} = \begin{bmatrix}
\bm{Z}^{\left(1\right)} & & \bm{0}\\
&\ddots &\\
\bm{0} & &\bm{Z}^{\left(m\right)}
\end{bmatrix}, \qquad 
\bm{\Lambda} &= \begin{bmatrix}
\bm{H} & & \bm{0}\\
&\ddots &\\
\bm{0} & &\bm{H}
\end{bmatrix}
\end{align*}
with $\bm{Z}^{\left(i\right)} \in \mathbb{R}^{n_i \times d}$, $\bm{H} \in \mathbb{R}^{d\times d}$, and $q = m d$.

\paragraph{Computational challenges} 
Parameter inference in LLMs requires accurately recovering $\mathcal{P} \coloneqq \left\{\bm{\beta},\bm{\Lambda},\sigma^2\right\}$ from $\left\{\bm{X},\bm{y},\bm{Z}\right\}$. This is straightforward if $\bm{\Lambda}$ is given. When $\bm{\Lambda}$ is unknown, inference can be computationally challenging even in the standard setting where $n > p$ \cite{Patterson71,Laird82,Laird87,Lindstrom88,Zhang08,Lippert11,Zhang11,Zhou17}.

First, parameter estimation problem is non-convex for both maximum likelihood and restricted maximum likelihood (REML) \cite{Patterson71,Harville74,Laird87}. For instance, methods using REML \cite{Kang08,Lippert11} project the data onto two uncorrelated parts, and then estimate the fixed-effects and variance components separately on each part. This has the advantage of providing unbiased estimates of the variance components. However, the REML likelihood function is a non-convex function which involves the eigenvalues of the variance of the projected data \cite{Patterson71}.

Second, regularization is typically required to support the high-dimensional setting, which adds further computational overhead \cite{Lippert11,Schelldorfer11,Schelldorfer14,Jakubik15,Zhou17}. To address these challenges, we develop novel approximate estimators that are efficient to compute (\cref{sec:kernel-comp}), and have provable accuracy guarantees (\cref{sec:guarantees}).

\section{Approximate Estimators for High-Dimensional LMMs}
\label{sec:estimators}
In this section, we consider $\ell_2$-regularized LMMs to address the high-dimensional setting $p > n$, and develop efficient approximate estimators for the parameters.

Standard parameter estimation algorithms for LMMs such as \cite{Laird87,Kang08,Bates15} do not support the high-dimensional setting $p > n$. A standard correction to extend LMMs to the high-dimensional setting is to introduce
$\ell_2$ regularization on the fixed-effects coefficients, which can be viewed as adding the prior $\bm{\beta} \sim \mathcal{N}\left(\bm{0},\bm{\Phi}\right)$. The $\ell_2$-regularized LMM has the following log-likelihood
\begin{align}
\begin{split}
\log p\left(\bm{y}, \bm{\beta} \mid \bm{X} ; \bm{V}\right)
\propto -\frac{1}{2} \bm{\beta}^\top \bm{\Phi}^{-1} \bm{\beta} - \frac{1}{2} \log\det\bm{V} -\frac{1}{2}  \left(\bm{y} - \bm{X}\bm{\beta} - c\bm{1}\right) \bm{V}^{-1} \left(\bm{y} - \bm{X}\bm{\beta} - c\bm{1}\right)
\end{split}
\label{eq:lmm-log-lik}
\end{align}
with the marginal variance $\bm{V} \coloneqq \bm{Z}\bm{\Lambda}\bm{Z}^\top + \sigma^2\bm{I}$.

Parameter estimation of an LMM is typically iterative and computationally prohibitive, especially in the high-dimensional setting \cite{Schelldorfer11,Darnell17,Perry17}. To improve the computational efficiency, we propose an approximate estimator that makes use of a dual representation. The approximation algorithm efficiently estimates the dual representation. These estimators are non-iterative and have reduced computational complexity, as we will show in \cref{sec:kernel-comp}.

\subsection{Fixed-effects coefficients}
We first derive the estimators for the fixed-effects coefficients $\widehat{\bm{\beta}}$ and $\widehat{c}$, which are the maximizers of the log-likelihood \eqref{eq:lmm-log-lik}. A dual estimator of $\bm{\beta}$ is then stated that is useful in the high-dimensional setting. Using the partial derivatives, it is straightforward to show
\begin{gather}
\left(\bm{X}^\top\bm{V}^{-1}\bm{X} + \bm{\Phi}^{-1}\right)\widehat{\bm{\beta}} = \bm{X}^\top\bm{V}^{-1}\left(\bm{y} - \widehat{c}\bm{1}\right) \label{eq:beta-primal}\\
\widehat{c} = \frac{\bm{1}^\top\bm{V}^{-1}\bm{y} - \bm{1}^\top\bm{V}^{-1}\bm{X}\widehat{\bm{\beta}}} {\bm{1}^\top\bm{V}^{-1}\bm{1}}. \label{eq:c-est}
\end{gather}
Let $\bm{L} = \bm{I} - \bm{1}\bm{1}^\top\bm{V}^{-1}\left(\bm{1}^\top\bm{V}^{-1}\bm{1}\right)^{-1}$, we obtain
\begin{align}
\label{eq:opt-beta}
\widehat{\bm{\beta}} = \left(\bm{X}^\top\bm{V}^{-1}\bm{L}\bm{X} + \bm{\Phi}^{-1}\right)^{-1}\bm{X}^\top\bm{V}^{-1}\bm{L}\bm{y}.
\end{align}
The dual estimator using $\bm{X}\bm{\Phi}\bm{X}^\top$ was proposed in \cite{Saunders98} where the authors used Lagrange multipliers to obtain the following estimator for ridge regression
\begin{align*}
\widehat{\bm{\beta}}_{\text{Dual}} = \bm{\Phi}\bm{X}^\top\left(\bm{V} + \bm{X}\bm{\Phi}\bm{X}^\top\right)^{-1}\bm{y}.
\end{align*}
Here, $\bm{\Phi}$ is set to be diagonal, and the above estimator \eqref{eq:beta-est-dual} can be evaluated in $O\left(n^2 p\right)$ time, a significant improvement when $p \gg n$. However, the computational bottleneck is evaluating the {\em kernel matrix} $\bm{X}\bm{\Phi}\bm{X}^\top$.

For the zero intercept case $\widehat{c} = 0$, the dual estimator \eqref{eq:beta-est-dual} is equivalent to \eqref{eq:opt-beta} by the following variant of the Woodbury identity $\left(\bm{U}^{-1} + \bm{A}^\top\bm{V}^{-1}\bm{A}\right)^{-1}\bm{A}^\top\bm{V}^{-1} = \bm{U}\bm{A}^\top\left(\bm{A}\bm{U}\bm{A}^\top + \bm{V}\right)^{-1}$ for invertible matrices $\bm{U}$ and $\bm{V}$. The dual estimator can be generalized to any intercept,
\begin{align}
\label{eq:beta-est-dual}
\widehat{\bm{\beta}} = \bm{\Phi}\bm{X}^\top\bm{V}^{-1}\bm{L}\left(\bm{X}\bm{\Phi}\bm{X}^\top\bm{V}^{-1}\bm{L} + \bm{I}\right)^{-1}\bm{y}.
\end{align}
Computing the dual estimator \eqref{eq:beta-est-dual} takes $O\left(n^2 p\right)$ time as opposed to $O\left(p^3\right)$ time required by \eqref{eq:opt-beta}. This complexity will be further improved in \cref{sec:kernel-comp} for the setting $p \gg n$.

\subsection{Approximate variance components}

The variance components $\bm{\Lambda}$ and $\sigma^2$ are typically estimated using an iterative EM algorithm with a per-iteration cost $O\left(p^3\right)$ \cite{Laird87,Lindstrom88} or an exhaustive grid search for the solution of a system of eigenvalue equations \cite{Kang08,Lippert11}. We consider an approximate non-iterative estimator based on the key observation that the optimization of the \eqref{eq:lmm-log-lik} has a simple closed-form solution if carried out with respect to $\bm{M} = \bm{V} + \bm{X}\bm{\Phi}\bm{X}^\top$. We will estimate $\bm{M}$ and use it as a proxy for estimating $\bm{\Lambda}$ as well as $\sigma^2$. The variance components inferred using $\bm{M}$ are referred to as the {\em approximate variance components} (AVCs). While AVCs may be used as variance components estimates under certain circumstances, their main purpose is to be an efficiently computable quantity used to estimate the fixed-effects coefficients. 

\paragraph{Proxy component estimation}

To perform the REML estimation of the variance components in terms of $\bm{M}$, we first rewrite the log-likelihood \eqref{eq:lmm-log-lik} as
\begin{align}
\label{eq:lik-decomp}
\begin{split}
l\left(\bm{\beta},\bm{V}\right)
&= -\frac{1}{2} \log\det\bm{V} -\frac{1}{2} \left(\bm{y} - \widehat{c}\bm{1}\right)^\top \bm{M}^{-1} \left(\bm{y} - \widehat{c}\bm{1}\right) -\frac{1}{2} \left(\bm{\beta} - \widehat{\bm{\beta}}\left(\bm{V}\right)\right)^\top \bm{Q} \left(\bm{\beta} - \widehat{\bm{\beta}}\left(\bm{V}\right)\right)
\end{split}
\end{align}
with
\begin{gather*}
\bm{Q} = \bm{X}^\top \bm{V}^{-1} \bm{X} + \bm{\Phi}^{-1} \\
\widehat{\bm{\beta}}\left(\bm{V}\right) = \left(\bm{X}^\top \bm{V}^{-1} \bm{X} + \bm{\Phi}^{-1}\right)^{-1} \bm{X}^\top \bm{V}^{-1} \left(\bm{y} - \widehat{c}\bm{1}\right).
\end{gather*}
Here, the estimate $\widehat{\bm{\beta}}$ depends on $\bm{V}$, and is consistent with the estimate given by \eqref{eq:opt-beta}. The $\widehat{c}$ in \eqref{eq:lik-decomp} can be set to the mean response, or estimated based on a prior distribution as in \cite{Zhou13}.

The REML estimator for the variance components is based on marginalizing the fixed effects $\bm{\beta}$ \cite{Harville74}. It follows that
\begin{align*}
l_p\left(\bm{V}\right) &\propto \log \int_{\mathbb{R}^p} \exp \left(l\left(\bm{\beta},\bm{V}\right)\right) d \bm{\beta}\\
&\propto -\frac{1}{2} \log\det\bm{V} -\frac{1}{2} \log\det\bm{Q} -\frac{1}{2} \left(\bm{y}-\widehat{c}\bm{1}\right)^\top \bm{M}^{-1}\left(\bm{y}-\widehat{c}\bm{1}\right).
\end{align*}
From Sylvester's determinant theorem, one observes that $\det\left(\bm{M}\right) = \det\left(\bm{\Phi}\right)\det\left(\bm{V}\right)\det\left(\bm{Q}\right)$. Thus, we arrive at
\begin{align}
\label{eq:v-reml}
\begin{split}
l_p\left(\bm{V}\right) &\propto -\frac{1}{2}  \log\det\bm{M} -\frac{1}{2}  \left(\bm{y}-\widehat{c}\bm{1}\right)^\top \bm{M}^{-1}\left(\bm{y}-\widehat{c}\bm{1}\right).
\end{split}
\end{align}

What we have achieved through \eqref{eq:v-reml} is a simple closed-form REML estimate of $\bm{V}$, rather than the non-convex or iterative updates for $\widehat{\bm{\Lambda}}$ and $\widehat{\sigma}^2$ required in state-of-the-art LMM parameter estimation algorithms. Unconstrained maximization of \eqref{eq:v-reml} with respect to $\bm{M}$ results in the closed-form equality 
\begin{align}
\bm{Z}\widehat{\bm{\Lambda}}\bm{Z}^\top + \widehat{\sigma}^2 \bm{I} = \left(\bm{y}-\widehat{c}\bm{1}\right)\left(\bm{y}-\widehat{c}\bm{1}\right)^\top - \bm{X}\bm{\Phi}\bm{X}^\top, \label{eq:proxy-sol}
\end{align}
for an optimal $\bm{M}$. Note that $\bm{Z}\widehat{\bm{\Lambda}}\bm{Z}^\top$ is positive semidefinite, whereas the right hand side has at most one positive eigenvalue. Thus, this optimal $\bm{M}$ may not be achievable and the unbiased estimate of $\bm{\Lambda}$ may possibly have negative eigenvalues. The issue of negative variance estimates in linear mixed models is an open problem \cite{Demidenko13} and beyond the scope of this paper. One resolution is to introduce a Gamma prior on $\bm{\Lambda}$ \cite{Chung13}. For unbiased estimation, we allow $\bm{\Lambda}$ to have negative eigenvalues, and intuitively we refer to the variance estimators obtained this way as {\em approximate variance components}.

\paragraph{Approximate variance estimators}
Assume that $\bm{Z}$ has full column rank and let 
\begin{align*}
\bm{S} = \left(\bm{y}-\widehat{c}\bm{1}\right)\left(\bm{y}-\widehat{c}\bm{1}\right)^\top - \bm{X}\bm{\Phi}\bm{X}^\top.
\end{align*} 
The approximate variance components $\widehat{\bm{\Lambda}}_{\text{AVC}}$ and $\widehat{\sigma}_{\text{AVC}}^2$ can be obtained via the following minimization problem
\begin{align}
\label{eq:fro-est}
\arg\min_{\bm{\Lambda}, \sigma^2} \left\|\bm{Z}\bm{\Lambda}\bm{Z}^\top - \bm{S} + \sigma^2\bm{I}\right\|_F^2.
\end{align}
Optimizing with respect to $\bm{\Lambda}$ yields 
\begin{align}
\label{eq:k-est-sig}
\bm{\Lambda}_\star = \bm{Z}^\dagger\left(\bm{S} - \sigma^2\bm{I}\right)\bm{Z}^{\dagger\top},
\end{align}
where $\bm{Z}^\dagger \coloneqq \left(\bm{Z}^\top\bm{Z}\right)^{-1}\bm{Z}^\top$. The estimators are computed by substituting $\bm{\Lambda}_\star$ into \eqref{eq:fro-est} and optimizing with respect to $\sigma^2$:
\begin{align}
\label{eq:ivc-est}
\begin{split}
\widehat{\sigma}_{\text{AVC}}^2 = \frac{\tr\left[\bm{S} \left(\bm{I} - \bm{Z}\bm{Z}^\dagger \right)\right]} {n - q} \qquad \text{and} \qquad
\widehat{\bm{\Lambda}}_{\text{AVC}} = \bm{Z}^\dagger \bm{S} \bm{Z}^{\dagger\top}
- \widehat{\sigma}_{\text{AVC}}^2 \left(\bm{Z}^\top\bm{Z}\right)^{-1}.
\end{split}
\end{align}

For a more robust estimator instead of optimizing one can compute the expectation under prior distributions for the objective in \eqref{eq:fro-est}, or use certain parameterizations of the variances. For example, consider the parameterization $\bm{\Lambda} = \theta \bm{D}$ in \cite{Kang08,Lippert11} with a fixed symmetric positive semi-definite $\bm{D}$, the solution to \eqref{eq:fro-est} is written as
\begin{align}
\label{eq:k-opt}
\bm{\Lambda}_* = \frac{\tr\left(\bm{G}\left(\bm{S} - \sigma^2\bm{I}\right)\right)}{\tr \left(\bm{G}^2\right)} \bm{D} \qquad \text{with} \qquad \bm{G} = \bm{Z}\bm{D}\bm{Z}^\top.
\end{align}
Substituting into \eqref{eq:fro-est}, we obtain
\begin{align}
\label{eq:sigma-est}
\widehat{\sigma}_{\text{AVC}}^2 = \frac{1}{n - \alpha}\left[\tr\left(\bm{S}\right) - \frac{\tr\left(\bm{G}\bm{S}\right)}{\tr\left(\bm{G}^2\right)}\right],
\end{align}
where $\alpha = \tr\left(\bm{G}\right)^2 / \tr\left(\bm{G}^2\right)$. Combined with \eqref{eq:k-opt}, we arrive at
\begin{align}
\label{eq:k-est}
\widehat{\bm{\Lambda}}_{\text{AVC}} = \frac{\tr\left(\bm{G}\left(\bm{S} - \widehat{\sigma}_{\text{AVC}}^2\bm{I}\right)\right)}{\tr \left(\bm{G}^2\right)} \bm{D}.
\end{align}

We use the AVCs to speed up estimating the fixed-effects coefficients. The complexity for computing the AVCs is $O\left(n^3\right)$, if $\bm{S}$ is given. Like the dual fixed-effects estimator \eqref{eq:beta-est-dual}, the computational bottleneck fpr AVCs also lies in evaluating $\bm{X}\bm{\Phi}\bm{X}^\top$.

\section{Fast Computational Algorithms}
\label{sec:kernel-comp}
In this section, we further improve the computational complexity $O\left(n^2 p\right)$ of the proposed approximate estimators in the high-dimensional setting $p \gg n$, where the computation bottleneck lies in evaluating the kernel $\bm{X}\bm{\Phi}\bm{X}^\top$.
We adopt the subsampled randomized Hadamard transform (SRHT) \cite{Tropp11} to compute the kernel matrix efficiently. In particular, the high-dimensional data is first projected into lower dimensions using the SRHT, and the parameters of the LMM are then estimated using the projected data. However, there are two main challenges involved: 1) the estimated parameter $\widehat{\bm{\beta}}$ has the reduced dimension of the projected data of reduced dimensions rather than the dimensionality of the original covariates; and 2) the impact of applying the SRHT on the accuracy of parameter estimation needs to be justified. The techniques developed in this section recovers the coefficients for all the covariates from the SRHT projected data with high accuracy, as will be shown in \cref{sec:guarantees}.

\subsection{Non-iterative algorithm for general LMMs}
\label{sec:fixedLMM}

In this subsection, we provide a fast algorithm for parameter estimation for a general covariance matrix. \cref{alg:approx} takes as input the matrices $\bm{X}$ and $\bm{\Phi}$ (which will be typically diagonal) and an approximation error $\epsilon$ described in \cref{sec:guarantees}. Both an approximate kernel matrix
$\bm{X}\bm{\Phi}\bm{X}^\top$ and the SRHT matrix $\bm{\Pi}$ are computed. The computational efficiency of the algorithm is a result of replacing $\bm{X}$ with the smaller transformed $\bm{A}$ in subsequent operations. Additionally, the structure of the SRHT allows for a divide-and-conquer scheme to compute $\bm{A} = \bm{X} \sqrt{\bm{\Phi}} \bm{\Pi}^\top$ in $O\left(n p \log p\right)$ time. Note that the matrix $\bm{W}_{p^\prime}$ is not formed explicitly. The computation $\bm{A}\bm{A}^\top$ requires $O\left(n^2 s_\epsilon\right)$ time, which becomes dominant setting $\epsilon \leq C n \sqrt{\frac{\log n}{p \log p}}$ for some universal constant $C$. Thus, the overall runtime for the algorithm is $O\left(\frac{n^3 \log n}{\epsilon^2} \right)$ for dense full-rank $\bm{X}$, and will be faster if $\bm{X}$ is low rank. The quality of the approximation depends on $\epsilon$, which will be discussed in \cref{sec:guarantees}.

\begin{algorithm}[tb]
\caption{Approximate kernel matrix computation.}
\begin{algorithmic}[1]
\label{alg:approx}
\REQUIRE {$\bm{X}$, $\bm{\Phi}$, and error tolerance $\epsilon$. }
\STATE {Let $p^\prime = 2^{\ceil*{\log_2 p}}$, append $p^\prime - p$ all zero columns to $\bm{X}$, and $p^\prime - p$ all zero rows and columns to $\bm{\Phi}$. Compute a diagonal matrix $\bm{D}$ of dimension $p^\prime$ with Rademacher random diagonal elements. }
\STATE {Denote the fast Walsh-Hadamard transform by
\begin{align*}
\bm{W}_{p^\prime} = \begin{bmatrix}
\bm{W}_{p^\prime/2} & \bm{W}_{p^\prime/2}\\
\bm{W}_{p^\prime/2} & - \bm{W}_{p^\prime/2}
\end{bmatrix}  \quad \text{with} \quad \bm{W}_1 = 1. 
\end{align*}
Let $r$ be the rank of $\bm{X}$ or $r = n$ for unknown rank, then define
\begin{align*}
s_{\epsilon} \coloneqq \frac{6 \left[\sqrt{r} + \sqrt{8 \log \left(r p^\prime\right)}\right]^2 \log r}{\epsilon^2}.
\end{align*}
Sample without replacement $m$ rows of $\bm{W}_{p^\prime}\bm{D} / \sqrt{s_{\epsilon}}$ to obtain the SRHT $\bm{\Pi}$. Compute the transformed covariate matrix $\bm{A} = \bm{X} \sqrt{\bm{\Phi}} \bm{\Pi}^\top$.
}
\STATE{{\bf return} the approximate kernel $\bm{A}\bm{A}^\top$, $\bm{A}$, and $\bm{\Pi}$. }
\end{algorithmic}
\end{algorithm}

Given the approximate kernel, it is straight forward to compute the AVCs $\bm{\Lambda}_{\text{AVC}}$ and $\sigma_{\text{AVC}}^2$ via \eqref{eq:ivc-est}. The coefficients for the fixed-effects can also be computed efficiently using the following estimator
\begin{align}
\widehat{\bm{\beta}}
= \sqrt{\bm{\Phi}}\bm{\Pi}^\top\bm{A}^\top\widehat{\bm{V}}^{-1}\bm{L}\left(\bm{I} + \bm{A}\bm{A}^\top\widehat{\bm{V}}^{-1}\bm{L}\right)^{-1}\bm{y}.
\label{eq:est-beta-fast}
\end{align}
Given the approximate kernel matrix and $\bm{A}$, computing $\bm{A}^\top\widehat{\bm{V}}^{-1}\bm{L}\left(\bm{I} + \bm{A}\bm{A}^\top\widehat{\bm{V}}^{-1}\bm{L}\right)^{-1}\bm{y}$ takes time $O\left(\max\left\{n^2 s_{\epsilon}, n^3\right\}\right)$ and multiplication of this vector by $\sqrt{\bm{\Phi}}\bm{\Pi}^\top$ is
$O\left(p \log p \right)$ due to the structure of the SRHT matrix as well as the fact that $\sqrt{\bm{\Phi}}$ is diagonal. The resulting complexity in computing \eqref{eq:est-beta-fast} is $O\left(\max\left\{n^2 s_{\epsilon}, n^3, p\log p\right\}\right)$.

Approximating the kernel matrix using the SRHT was proposed for ridge regression in \cite{Lu13}, a special case of our setting. A method for estimating the full set of fixed-effects coefficients was not provided in \cite{Lu13}. Instead, a reduced set of $m$ fixed-effects coefficients corresponding to the transformed covariates $\bm{X} \bm{\Pi}^\top$ was computed. For many applications, a major point of using an LMM is to estimate the effect-size of the fixed-effects coefficients, so computing $\widehat{\bm{\beta}}$ is essential to the problem.

\subsection{Fast EM for multi-group LMMs}
\label{sec:generalization}

For efficient parameter estimation in $ell_2$-regularized LMMs with repeated measurements, we extend the EM algorithm for the low-dimensional setting $n \geq p$ \cite{Laird87} by combing the kernel estimators and \cref{alg:approx}. While this high-dimensional EM variant is iterative, we show that the per-iteration computational cost is scalable in $p$.

The log-likelihood of the $\ell_2$-regularized LMM \eqref{eq:lmm2-log-lik} can be rewritten in terms of class-specific variables as
\begin{align}
\begin{split}
\log p\left(\bm{y}, \bm{\gamma}, \bm{\beta} \mid \bm{X} ; \sigma^2,\bm{\Lambda}\right)
&\propto -\frac{1}{2} \bm{\beta}^\top \bm{\Phi}^{-1} \bm{\beta}
-\frac{n}{2} \log \sigma^2 - \frac{m}{2} \log\det\bm{H}\\
&\phantom{=} -\frac{1}{2} \sum_{i=1}^m \bm{\gamma}^{\left(i\right)\top}\bm{H}^{-1}\bm{\gamma}^{\left(i\right)}
-\frac{\bm{e}^\top\bm{e}}{2\sigma^2},
\end{split}
\label{eq:lmm2-log-lik}
\end{align}
where $\bm{e} = \bm{y} - c\bm{1} - \bm{X}\bm{\beta} - \bm{Z}\bm{\gamma}$.

From the above log-likelihood, the distribution of $\bm{\beta}$ conditioned on the data and parameter estimates $\widehat{\mathcal{P}} \coloneqq \left\{\widehat{c},\widehat{\sigma}^2,\widehat{\bm{H}}\right\}$ is is multivariate normal with mean $\bm{\Phi}\bm{X}^\top\widehat{\bm{M}}^{-1}\left(\bm{y} - \widehat{c}\bm{1}\right)$ and covariance $\bm{\Phi} - \bm{\Phi}\bm{X}^\top\widehat{\bm{M}}^{-1}\bm{X}\bm{\Phi}$. Similarly, the posterior distribution of the vector of latent variables $\bm{\gamma}$ is multivariate normal with mean $\widehat{\bm{\Lambda}}\bm{Z}^\top\widehat{\bm{M}}^{-1}\left(\bm{y} - \widehat{c}\bm{1}\right)$ and covariance $\widehat{\bm{\Lambda}} - \widehat{\bm{\Lambda}}\bm{Z}^\top\widehat{\bm{M}}^{-1}\bm{Z}\widehat{\bm{\Lambda}}$. Denote by $\widehat{\bm{\gamma}}$ the mean of the posterior distribution of $\bm{\gamma}$, we also obtain the following posterior distributions of class-specific latent variables $\bm{\gamma}^{\left(i\right)}$:
\begin{align}
\label{eq:posterior-gamma}
\mathcal{N}\left(\widehat{\bm{\gamma}}^{\left(i\right)},
\widehat{\bm{H}} - \widehat{\bm{H}}\bm{Z}^{\left(i\right)\top}\left(\widehat{\bm{M}}^{-1}\right)^{\left(i\right)}\bm{Z}^{\left(i\right)}\widehat{\bm{H}}\right).
\end{align}
Note that $\cdot^{\left(i\right)}$ represents the block matrix corresponding to group $i$.
These posteriors are used in the E-step, discussed next.

\paragraph{E-step} In the E-step, we derive the expectation of the log-likelihood \eqref{eq:lmm2-log-lik} with respect to the aforementioned posterior distribution of $\bm{\beta}$ and $\bm{\gamma}^{\left(i\right)}$:
\begin{align*}
\mathbb{E}_{\bm{\beta},\bm{\gamma} \mid \bm{y},\widehat{\mathcal{P}}} \left[
\log p\left(\bm{y}, \bm{\gamma}, \bm{\beta} \mid \bm{X} ; \sigma^2,\bm{H}\right)
\right].
\end{align*}
We only need to consider terms in the expectation that involve $c$, $\sigma^2$, and $\bm{H}$. Denote by $\widehat{\Sigma}_{\bm{\gamma}^{\left(i\right)}}$ the variance of \eqref{eq:posterior-gamma}, the following holds  $\mathbb{E}_{\bm{\beta},\bm{\gamma} \mid \bm{y},\widehat{\mathcal{P}}}\left(\bm{\gamma}^{\left(i\right)\top}\bm{H}^{-1}\bm{\gamma}^{\left(i\right)}\right) =
\widehat{\bm{\gamma}}^{\left(i\right)\top}\bm{H}^{-1}\widehat{\bm{\gamma}}^{\left(i\right)} +
\tr \left(\widehat{\bm{\Sigma}}_{\bm{\gamma}^{\left(i\right)}}\bm{H}^{-1}\right)$. Using the previously derived posterior distributions, we get $\mathbb{E}\left(\bm{e} \mid \bm{y}, \widehat{\mathcal{P}} \right) = \bm{y} - \bm{X}\widehat{\bm{\beta}} - \bm{Z}\widehat{\bm{\gamma}} - \widehat{c}\bm{1}$ and
\begin{align*}
\text{cov}\left(\bm{e} \mid \bm{y}, \widehat{\mathcal{P}} \right) 
&= \bm{X}\left(\bm{\Phi} - \bm{\Phi}\bm{X}^\top \widehat{\bm{M}}^{-1} \bm{X}\bm{\Phi}\right)\bm{X}^\top +
\bm{Z}\left( \widehat{\bm{\Lambda}} - \widehat{\bm{\Lambda}}\bm{Z}^\top \widehat{\bm{M}}^{-1} \bm{Z}\widehat{\bm{\Lambda}}\right)\bm{Z}^\top\\
&= \widehat{\bm{\Sigma}} - \widehat{\bm{\Sigma}}\left(\widehat{\bm{\Sigma}} + \widehat{\sigma}^2\bm{I}\right)^{-1}\widehat{\bm{\Sigma}}\\
&= \left(\bm{\Sigma}^{-1} + \widehat{\sigma}^{-2}\bm{I}\right)^{-1}\\
&= \widehat{\sigma}^2\bm{I} - \widehat{\sigma}^4\widehat{\bm{M}}^{-1}
\end{align*}
with $\widehat{\bm{\Sigma}} \coloneqq \bm{X}\bm{\Phi}\bm{X}^\top + \bm{Z}\widehat{\bm{\Lambda}}\bm{Z}^\top$. Thus, we arrive at $\mathbb{E}_{\bm{\beta},\bm{\gamma} \mid \bm{y},\widehat{\mathcal{P}}}\left(\bm{e}^\top\bm{e}\right) = \widehat{\bm{e}}^\top\widehat{\bm{e}} + \widehat{\sigma}^2\bm{I} - \widehat{\sigma}^4\widehat{\bm{M}}^{-1}$, where $\widehat{\bm{e}} \coloneqq \mathbb{E}\left(\bm{e} \mid \bm{y},\widehat{\mathcal{P}} \right)$.

\paragraph{M-step} We now update the parameter estimates by maximizing the expectation from the E-step. 
First, observe that the $\bm{\beta}$ estimate from the posterior distribution is the same as the the dual estimator developed in \cref{sec:estimators}. To maximize the expectation with respect to $\bm{H}$ and $\sigma^2$, we take the partial derivatives with respect to $\bm{H}^{-1}$ and $\sigma^{-2}$, and set them to zero. This gives the following M-step updates
\begin{gather}
\label{eq:em-var-est}
\begin{split}
\widehat{\bm{H}} \leftarrow \frac{1}{m}\sum_{i=1}^m \left(
\widehat{\bm{\gamma}}^{\left(i\right)}\widehat{\bm{\gamma}}^{\left(i\right)\top} 
+ \widehat{\Sigma}_{\bm{\gamma}^{\left(i\right)}}\right), \qquad
\widehat{\sigma}^2 \leftarrow \widehat{\sigma}^2 + \frac{1}{n} \left[ \widehat{\bm{e}}^\top\widehat{\bm{e}} - \widehat{\sigma}^4 \tr \left(\widehat{\bm{M}}^{-1}\right)\right].
\end{split}
\end{gather}

The fast version of the above EM algorithm uses \cref{alg:approx} for computing the kernel. Note that the original $\bm{X}$ is no longer needed after the SRHT projection. This provides additional space advantages as the data $\bm{X}$ can be preprocessed, and the Hadamard transform in Step $1$ requires a small constant amount of memory. Overall, the per-iteration computational complexity of the EM algorithm is $O\left(\max\left\{n^2 s_{\epsilon}, n^3\right\}\right)$.

\section{Theoretical Guarantees}
\label{sec:guarantees}

In this section, we provide an analysis of the difference in the parameters estimated via the approximate algorithms versus minimizing the $\ell_2$-regularized LMM. We are not proving consistency of our estimator---convergence of the parameter estimates to the population quantity. Consistency results for LMMs and-regularized LMMs were provided in \cite{Hall03,Cui04,Schelldorfer11}.

\begin{theorem} [Fixed-effects norm error]
\label{thm:beta-approx}
Let $\widehat{\bm{\beta}}$ be the fixed-effects coefficients estimated by \eqref{eq:opt-beta} and $\widehat{\bm{\beta}}^\prime$ be the fixed-effects coefficients estimated by the approximate procedure in \eqref{eq:est-beta-fast}. Then, with probability at least $1-3/n$ \begin{align*}
\frac{\left\|\widehat{\bm{\beta}} - \widehat{\bm{\beta}}^\prime\right\|}{\left\|\widehat{\bm{\beta}}\right\|}
\leq
\frac{\epsilon}{1 - \epsilon}
\frac{\left\|\bm{\Phi}^{-1}\right\|_2 \kappa\left(\bm{\Gamma}\right)}{\frac{\left\|\bm{\Phi}\right\|_2^{-1}}{1 + \sqrt{2/3}\epsilon} + \lambda_{\text{min}}\left(\bm{X}^\top\bm{V}^{-1}\bm{X}\right)}
\end{align*}
with $\bm{\Gamma} \coloneqq \bm{\Phi}^{-1} + \bm{X}^\top\bm{V}^{-1}\bm{X}$, or loosely
\begin{align*}
\frac{\left\|\widehat{\bm{\beta}} - \widehat{\bm{\beta}}^\prime\right\|} {\left\|\widehat{\bm{\beta}}\right\|}
\leq \frac{\epsilon\left(1 + \sqrt{2/3}\epsilon\right)}{1 - \epsilon} \kappa\left(\bm{\Phi}\right)
\kappa\left(\bm{\Gamma}\right) 
\end{align*}
for all $0 \leq \epsilon < 1$.
\end{theorem}

The proofs use the following two results, i.e., \cref{lem:row-norm} and \cref{lem:extreme-singular}, of subsampled randomized Hadamard transform (SRHT).
\begin{lemma}[\cite{Tropp11}]
\label{lem:row-norm}
Suppose that $\bm{V}$ is an $n \times k$ matrix with orthonormal columns and $\bm{\Pi}$ is an $n \times n$ SRHT matrix, it satisfies
\begin{align*}
\Pr\left(\max_j \left\|\bm{e}_j^\top \bm{\Pi} \bm{V}\right\| \geq \sqrt{\frac{k}{n}} + \sqrt{\frac{8\log \left(n/\delta\right)}{n}}\right)  \leq \delta.
\end{align*}
\end{lemma}

\begin{lemma}[\cite{Tropp11}]
\label{lem:extreme-singular}
Let $\bm{V}$  be an $n \times k$ matrix with orthonormal columns, and denote the maximum squared row norm by $\gamma = \max_j \left\|\bm{e}_j^\top \bm{V}\right\|^2$. Sample uniformly without replacement $m$ rows of $\bm{V}$ to obtain a reduced matrix $\bm{V}^\prime$. For any $t > 0$, the extreme singular values satisfy
\begin{align*}
\sigma_1\left(\bm{V}^\prime\right) \leq \sqrt{\frac{\left(1+\alpha\right)m}{n}} \; \text{and} \;
\sigma_k\left(\bm{V}^\prime\right) \geq \sqrt{\frac{\left(1-\beta\right)m}{n}}
\end{align*}
with failure probability at most
\begin{align*}
k \left[\frac{e^\alpha}{\left(1+\alpha\right)^{1+\alpha}}\right]^{\frac{m}{n\gamma}} +
k \left[\frac{e^{-\beta}}{\left(1-\beta\right)^{1-\beta}}\right]^{\frac{m}{n\gamma}}.
\end{align*}
\end{lemma}

The following theorem is a consequence of \cref{lem:row-norm} and \cref{lem:extreme-singular}. The theorem basically states that kernel approximation in \cref{alg:approx} is close to the true kernel up to some scaling factor $1 \pm \epsilon$.

\begin{theorem} [Approximate matrix multiplication]
\label{thm:psd-pert}
Let $\bm{A}$ be an $n\times p$ matrix with rank $r$. Let $\bm{\Pi}$ be an $m \times p$ SRHT matrix with
\begin{align*}
m \geq \frac{6 \left[\sqrt{r} + \sqrt{8 \log \left(r p\right)}\right]^2 \log r}{\epsilon^2}.
\end{align*}
Suppose that $p > m$ and compute $\widehat{\bm{A}} = \bm{A}\bm{\Pi}^\top$, then the inequality 
\begin{align*}
\left(1 - \epsilon\right) \bm{A}\bm{A}^\top \preceq \widehat{\bm{A}}\widehat{\bm{A}}^\top \preceq \left(1 + \sqrt{\frac{2}{3}}\epsilon\right) \bm{A}\bm{A}^\top.
\end{align*}
fails with probability at most 3/n.
\end{theorem}

\begin{proof}
Note that the failure probability in \cref{lem:extreme-singular} is no more than
\begin{align*}
k \exp\left(-\frac{\alpha^2 m}{3 n \gamma}\right) + k \exp\left(-\frac{\beta^2 m}{2 n \gamma}\right).
\end{align*}
To make the failure probability no more than $2 k^{-1}$, it suffices to set
\begin{align*}
\alpha \geq \sqrt{\frac{6 n \gamma \log k}{m}} \qquad \text{and} \qquad
\beta \geq \sqrt{\frac{4 n \gamma \log k}{m}}.
\end{align*}
Incorporating the scaling factors of the SRHT, the extreme singular values of the transformed $\bm{V}$ satisfy
\begin{align*}
\sigma_1\left(\sqrt{\frac{n}{m}}\bm{V}^\prime\right) \leq \sqrt{1 + \sqrt{\frac{6 n \gamma \log k}{m}}} \qquad \text{and} \qquad
\sigma_k\left(\sqrt{\frac{n}{m}}\bm{V}^\prime\right) \geq \sqrt{1 - \sqrt{\frac{4 n \gamma \log k}{m}}}.
\end{align*}
From \cref{lem:row-norm}, with failure probability at most $k^{-1}$ that
\begin{align*}
\gamma \leq \left[ \sqrt{\frac{k}{n}} + \sqrt{\frac{8\log \left(n k\right)}{n}} \right]^2.
\end{align*}
Combined with the singular value bounds, this result establishes the connection between $m$ and the desired singular value bounds. One may choose $m = t n \gamma \log k$, for some $t \geq 4$.

The rest of the proof is straightforward. We consider the singular value decomposition $\bm{A} = \bm{U}\bm{\Sigma}\bm{V}^\top$, where $\bm{V}$ is $p \times r$ orthonormal $\bm{V}$ and has orthonormal columns. Let $\bm{\Pi}$ be the SRHT, we have that $\widehat{\bm{A}}\widehat{\bm{A}}^\top = \bm{U}\bm{\Sigma} \left(\bm{V}^\top \bm{\Pi}^\top \bm{\Pi}\bm{V}\right) \bm{\Sigma}\bm{U}^\top$. The desired result follows by invoking \cref{lem:extreme-singular} to bound the extreme singular values of $\bm{V}^\top \bm{\Pi}^\top \bm{\Pi}\bm{V}$.
\end{proof}

\begin{proof}[Proof of \cref{thm:beta-approx}]
The idea is to simplify the analysis by dealing with the equivalent primal form of \eqref{eq:est-beta-fast}, involving only one $\bm{\Pi}$ term. We then perform a perturbation analysis of the inverse component. In addition, Weyl's inequalities as well the exponentiated version of Horn's inequalities are used for eigenvalue manipulations.

First, \eqref{eq:est-beta-fast} can be equivalently expressed in the primal form \eqref{eq:beta-primal}:
\begin{align}
\label{eq:beta-approx-primal}
\begin{split}
\widehat{\bm{\beta}}^\prime 
&= \sqrt{\bm{\Phi}}\bm{\Pi}^\top\bm{\Pi}\sqrt{\bm{\Phi}}\bm{X}^\top\widehat{\bm{V}}^{-1}
 \left(\bm{I} + \bm{X}\sqrt{\bm{\Phi}}\bm{\Pi}^\top\bm{\Pi}\sqrt{\bm{\Phi}}\bm{X}^\top\widehat{\bm{V}}^{-1}\right)^{-1}\bm{y}\\
&= \left[\left(\sqrt{\bm{\Phi}}\bm{\Pi}^\top\bm{\Pi}\sqrt{\bm{\Phi}}\right)^{-1} + \bm{X}^\top\widehat{\bm{V}}^{-1}\bm{X}\right]^{-1} \bm{X}^\top\widehat{\bm{V}}^{-1}\bm{y}.
\end{split}
\end{align}
Let $\bm{\Gamma} = \bm{\Phi}^{-1} + \bm{X}^\top\widehat{\bm{V}}^{-1}\bm{X}$, the idea is to bound the error norm using the perturbation of the singular values of $\bm{\Gamma}^{-1}$. Denote by $\bm{\Phi}^\prime = \sqrt{\bm{\Phi}}\bm{\Pi}^\top\bm{\Pi}\sqrt{\bm{\Phi}}$ and $\bm{\Delta} = \bm{\Phi}^{\prime -1} - \bm{\Phi}^{-1}$, a basic result from matrix perturbation theory \cite{Stewart90} gives
\begin{align*}
\left\|\bm{\Gamma}^{-1} - \left(\bm{\Gamma} + \bm{\Delta}\right)^{-1}\right\|_2
&\leq \left\|\bm{\Delta}\right\|_2 \left\|\bm{\Gamma}^{-1}\right\|_2 \left\|\left(\bm{\Gamma} + \bm{\Delta}\right)^{-1}\right\|_2.
\end{align*}
From Weyl's inequalities, one further obtains
\begin{align*}
\left\|\left(\bm{\Gamma} + \bm{\Delta}\right)^{-1}\right\|_2
\leq \left[\lambda_{\text{min}}\left(\bm{\Phi}^{\prime-1}\right) + \lambda_{\text{min}}\left(\bm{X}^\top\widehat{\bm{V}}^{-1}\bm{X}\right) \right]^{-1}.
\end{align*}

We now provide a bound for $\left\|\bm{\Delta}\right\|_2$. Observe that $\bm{\Delta} = \sqrt{\bm{\Phi}}^{-1} \left(\left(\bm{\Pi}^\top\bm{\Pi}\right)^{-1} - \bm{I}\right) \sqrt{\bm{\Phi}}^{-1}$ in which the extreme singular values of the parenthesized difference are bounded via \cref{thm:psd-pert}. Thus, we have
\begin{align*}
\left\|\bm{\Delta}\right\|_2 &\leq \max\left\{\frac{\epsilon}{1 - \epsilon}, \frac{\sqrt{2/3} \epsilon}{1 + \sqrt{2/3}\epsilon} \right\}\left\|\bm{\Phi}^{-1}\right\|_2
= \frac{\epsilon}{1 - \epsilon}\left\|\bm{\Phi}^{-1}\right\|_2.
\end{align*}
It remains to give a lower bound for $\lambda_{\text{min}} \left(\bm{\Phi}^{\prime-1}\right)$. From \cref{thm:psd-pert} and Horn's inequalities, one has
\begin{align*}
\lambda_{\text{min}} \left(\bm{\Phi}^{\prime-1}\right) 
\geq \lambda_{\text{min}}\left(\left(\bm{\Pi}^\top\bm{\Pi}\right)^{-1}\right)
\lambda_{\text{min}}\left(\bm{\Phi}^{-1/2}\right)^2
= \frac{\left\|\bm{\Phi}\right\|_2^{-1}}{1 + \sqrt{2/3}\epsilon}.
\end{align*}
Finally, the desired estimation bound satisfies
\begin{align*}
\frac{\left\|\widehat{\bm{\beta}} - \widehat{\bm{\beta}}^\prime\right\|}{\left\|\widehat{\bm{\beta}}\right\|}
\leq \frac{\left\|\bm{\Gamma}^{-1} - \left(\bm{\Gamma} + \bm{\Delta}\right)^{-1}\right\|_2}{\sigma_{\text{min}}\left(\bm{\Gamma}^{-1}\right)}
\leq \frac{\epsilon}{1 - \epsilon}
\frac{\left\|\bm{\Phi}^{-1}\right\|_2 \kappa\left(\bm{\Gamma}\right) }{\frac{\left\|\bm{\Phi}\right\|_2^{-1}}{1 + \sqrt{2/3}\epsilon} + \lambda_{\text{min}}\left(\bm{X}^\top\widehat{\bm{V}}^{-1}\bm{X}\right)}.
\end{align*}
Setting $\lambda_{\text{min}}\left(\bm{X}^\top\widehat{\bm{V}}^{-1}\bm{X}\right) = 0$ yields the simplified worst-case bound
\begin{align*}
\left\|\widehat{\bm{\beta}} - \widehat{\bm{\beta}}^\prime\right\| \leq \frac{\epsilon\left(1 + \sqrt{2/3}\epsilon\right)}{1 - \epsilon} \kappa\left(\bm{\Phi}\right) \kappa\left(\bm{\Gamma}\right) \left\|\widehat{\bm{\beta}}\right\|.
\end{align*}
\end{proof}

An intuitive interpretation of the theorem is that the fixed-effects coefficients estimator \eqref{eq:est-beta-fast} has better accuracy when the predefined $\bm{\Phi}$ is better conditioned and has smaller spectral norm. One can certainly improve the accuracy by setting a smaller $\epsilon$, which in turn uses more observations in \cref{alg:approx}.

\begin{theorem} [AVC approximation errors]
\label{thm:var-approx}
Let $\sigma_{\text{AVC}}^2$ and $\widehat{\bm{\Lambda}}_{\text{AVC}}$ be computed using \eqref{eq:ivc-est}. Let $\widehat{\sigma}_{\text{AVC}}^{\prime 2}$ and $\widehat{\bm{\Lambda}}_{\text{AVC}}^\prime$ be computed using the same equations but with the approximate kernel from \cref{alg:approx}. Then, the following two statements hold jointly with probability at least $1 - 3/n$:
\begin{align*}
\left|\widehat{\sigma}_{\text{AVC}}^2 - \widehat{\sigma}_{\text{AVC}}^{\prime 2}\right| &\leq  \epsilon \cdot \frac{\vertiii{\bm{X}\bm{\Phi}\bm{X}^\top}_{n-q}}{n - q} \quad \text{and}\\
\left\|\widehat{\bm{\Lambda}}_{\text{AVC}} - \widehat{\bm{\Lambda}}_{\text{AVC}}^\prime\right\|_2
&\leq \frac{\epsilon}{\widehat{\sigma}_{\text{min}}\left(\bm{Z}\right)^2} \left(\left\| \bm{X}\bm{\Phi}\bm{X}^\top \right\|_2 + \frac{\vertiii{\bm{X}\bm{\Phi}\bm{X}^\top}_{n-q}}{n - q}\right).
\end{align*}
\end{theorem}
\begin{proof}
Let $\bm{P} = \bm{Z}\bm{Z}^\dagger$, then $\bm{I} - \bm{P}$ is idempotent. Thus, the noise AVC $\widehat{\sigma}_{\text{AVC}}^2$ in \eqref{eq:ivc-est} can be expressed as
\begin{align*}
\widehat{\sigma}_{\text{AVC}}^2 = \frac{\tr \left[\left(\bm{I} - \bm{P}\right) \bm{S} \left(\bm{I} - \bm{P}\right)\right]}{n - q}.
\end{align*}
The SRHT version $\widehat{\sigma}_{\text{AVC}}^{\prime 2}$ using \cref{alg:approx} satisfies
\begin{align*}
\left|\widehat{\sigma}_{\text{AVC}}^2 - \widehat{\sigma}_{\text{AVC}}^{\prime 2}\right|
= \left| \frac{\tr\left[ \left(\bm{I} - \bm{P}\right) \left(\bm{X}\bm{\Phi}\bm{X}^\top - \bm{A}\bm{A}^\top\right) \left(\bm{I} - \bm{P}\right) \right]} {n - q} \right|,
\end{align*}
where $\bm{A}$ is given in \cref{alg:approx}. One then invokes \cref{thm:psd-pert} to bound the singular values of $\bm{X}\bm{\Phi}\bm{X}^\top - \bm{A}\bm{A}^\top$:
\begin{align*}
\left|\widehat{\sigma}_{\text{AVC}}^2 - \widehat{\sigma}_{\text{AVC}}^{\prime 2}\right| 
\leq \epsilon \cdot \frac{\tr\left[ \left(\bm{I} - \bm{P}\right) \bm{X}\bm{\Phi}\bm{X}^\top \left(\bm{I} - \bm{P}\right) \right]} {n - q}
\leq \frac{\epsilon \sum_{i=1}^{n-q} \lambda_i\left(\bm{X}\bm{\Phi}\bm{X}^\top\right)}{n - q}
\end{align*}
fails with probability at most $3/n$. The second line follows from the exponentiated Horn's inequalities and the fact that $\bm{I} - \bm{P}$ is an idempotent projection matrix of rank $n-q$. The sum in the fraction equals to the Ky Fan $\left(n-q\right)$-norm of $\bm{X}\bm{\Phi}\bm{X}^\top$.

To prove the bound for $\widehat{\bm{\Lambda}}_{\text{AVC}}$, it follows from \eqref{eq:k-est} that
\begin{align*}
\left\|\widehat{\bm{\Lambda}}_{\text{AVC}} - \widehat{\bm{\Lambda}}_{\text{AVC}}^\prime\right\|_2 &\leq \left\|\bm{Z}^\dagger \left(\bm{X}\bm{\Phi}\bm{X}^\top - \bm{A}\bm{A}^\top \right) \bm{Z}^{\dagger\top} \right\|_2
+ \left|\widehat{\sigma}_{\text{AVC}}^2 - \widehat{\sigma}_{\text{AVC}}^{\prime 2}\right| \left\|\left(\bm{Z}^\top \bm{Z}\right)^{-1}\right\|_2\\
&\leq \frac{\epsilon}{\sigma_{\text{min}}\left(\bm{Z}\right)^2} \left(\left\| \bm{X}\bm{\Phi}\bm{X}^\top \right\|_2 + \frac{\vertiii{\bm{X}\bm{\Phi}\bm{X}^\top}_{n-q}}{n - q}\right),
\end{align*}
where we used \cref{thm:psd-pert} and the earlier bound on $\widehat{\sigma}_{\text{AVC}}^2$.
\end{proof}

Note that the fraction of the Ky Fan norm does not exceed the spectral norm. Looser but more convenient bounds are 
\begin{align*}
\left|\widehat{\sigma}_{\text{AVC}}^2 - \widehat{\sigma}_{\text{AVC}}^{\prime 2}\right| &\leq \epsilon \left\|\bm{X}\bm{\Phi}\bm{X}^\top\right\|_2 \\
\left\|\widehat{\bm{\Lambda}}_{\text{AVC}} - \widehat{\bm{\Lambda}}_{\text{AVC}}^\prime\right\|_2 &\leq 2 \epsilon \sigma_{\text{min}}\left(\bm{Z}\right)^{-2} \left\|\bm{X}\bm{\Phi}\bm{X}^\top\right\|_2.
\end{align*}

\section{Experiments}
\label{sec:experiments}
In this section, we conduct a simulation study as well as numerical experiments on real data. The simulation study demonstrates the accuracy of parameter estimation and the decreased runtime using the proposed Approximate Ridge LMM ({\tt arLMM}) methods. We also examined the results on real data from the Wellcome Trust Case Control Consortium (WTCCC) study \cite{WTCCC}, which include about 14,000 cases from seven common diseases and a total of about 450,000 SNPs.

The main finding of the experiments is that the proposed approximate inference algorithms enjoy similar predictive accuracy as state-of-the-art methods at a significantly reduced computation cost in practice. In particular, our Matlab prototype implementation is 6x faster than the optimized C implementation of the state-of-the-art {\tt BSLMM} method for genome-wide association studies.

\subsection{Simulation studies}
To evaluate parameter estimation, we consider two performance metrics. The first one is the correlation between the estimated and ground-truth fixed-effects coefficients. The second metric is the Negative Log Likelihood (NLL) of the standard LMM, which meaningfully reflects the quality of variance estimation.

For the simulation, we compare the performance of our non-iterative algorithm {\tt arLMM-AVC} based on \eqref{eq:est-beta-fast} and \eqref{eq:ivc-est}, the proposed multi-group variant {\tt arLMM-EM} based on \eqref{eq:em-var-est}, the standard {\tt REML} \footnote{Note that \cite{Kang08,Lippert11} are all specific implementation of the REML, but using different parameterizations to improve efficiency.} \cite{Bates15}, $\ell_1$-regularized LMM {\tt lmmlasso} \cite{Schelldorfer11}, and {\tt CovexLasso} using both $\ell_1$- and $\ell_2$-regularization \cite{Jakubik15}. The online implementation of these methods are used.

\paragraph{Synthetic data generation} The simulation is based on synthetic training and validation sets sampled from a fixed LMM distribution. The design matrices as well as the parameters for the fixed LMM are randomly generated. Specifically,
\begin{align*}
X_{ij} &\overset{\text{i.i.d.}}{\sim} \mathcal{N}\left(0,1\right) &
Z_{ij}^{\left(k\right)} &\overset{\text{i.i.d.}}{\sim} \mathcal{U}\left(0,1\right)
&\bm{\gamma}^{\left(k\right)} &\overset{\text{i.i.d.}}{\sim} \mathcal{N}\left(\bm{0}, \bm{K}^\top\bm{K}\right)\\
K_{ij} &\overset{\text{i.i.d.}}{\sim} \mathcal{N}\left(0,1\right)
&\bm{\beta} &\overset{\phantom{\text{i.i.d.}}}{\sim} \mathcal{N}\left(\bm{0},\bm{I}\right) &
\sigma^2 &\overset{\phantom{\text{i.i.d.}}}{\sim} \mathcal{U}\left(0,d\right).
\end{align*}
Note that there are $d$ random-effect variables with covariance $\bm{K}^\top\bm{K}$. Thus, the random-effect design matrix $\bm{Z}\in\mathbb{R}^{n\times q}$, $q=m d$, will be block-diagonal with diagonal blocks $\bm{Z}^{\left(k\right)}$. Given the number of observations $n$, we randomly sample $n_k$ observations for each group $k$, where the fractions $n_k/n$ are specified by the Dirichlet distribution with the concentration parameters $\left(1,1,\cdots\right)^\top$.

\paragraph{Overdetermined settings}
We first consider the standard setting $n > p$, which are supported by many parameter estimation algorithms of LMMs. We evaluate the performance of {\tt arLMM-AVC} and {\tt arLMM-EM} in a variety of $p$, $d$, and $m$ settings. The parameter estimates obtained using the proposed methods are compared with the estimates given by the standard REML (see e.g., \cite{Laird87,Kang08,Lippert11,Bates15}) which is known to produce unbiased estimates.

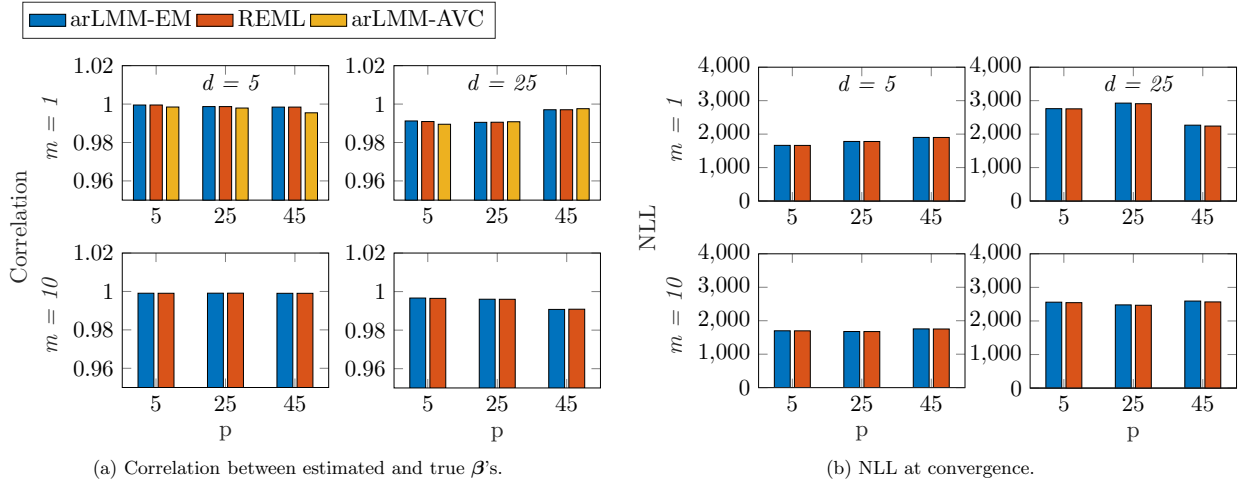
\begin{figure}[t]
\centering
\begin{adjustbox}{width=\columnwidth}
\subfloat[Correlation between estimated and true $\bm{\beta}$'s.]{
\input{corr}} \quad
\subfloat[NLL at convergence.]{
\input{nll}}
\end{adjustbox}
\caption{Comparing the performance of parameter estimation on synthetic data with $n=1,000$ observations. Note that {\tt arLMM-AVC} is only applicable to the single group setting. This figure shows that the {\tt arLMM-EM} and {\tt arLMM-AVC} achieve comparable estimation performance as {\tt REML}.}
\label{fig:low}
\end{figure}

\cref{fig:low} shows the error for the fitted parameters using 1,000 observations sampled from the underlying LMM. The average results are reported over 10 runs on independently generated datasets. These generated datasets have the same number of observations $n=1,000$ but different settings of $p$, $d$, and $m$.

As shown in \cref{fig:low}, {\tt arLMM-EM} and {\tt arLMM-AVC} exhibit comparable estimation accuracy as the standard {\tt REML}. Note that {\tt arLMM-AVC} is applicable only when $m=1$ (the first row of \cref{fig:low}). Since {\tt arLMM-AVC} is based on non-iterative approximation to the variance components, the error is slightly higher than the others as expected.

\paragraph{High-dimensional (underdetermined) setting}
We also examined the performance of our model in the high-dimensional setting where we are interested in variable selection based on the fixed-effects coefficients. In \cref{tbl:exp-confs}, we specify the three regimes for which we generate simulated data:
an overdetermined LMM, a moderate-dimensional LMM, and a high-dimensional LMM. Each regime is characterized by $n$, $p$, $d$, and $m$, and an extra parameter $s$, the number of non-zeros in the ground-truth $\bm{\beta}_{\text{True}}$. Since $m>1$ we did not apply {\tt arLMM-AVC}.

\begin{table}[htb]
\caption{Regimes of data.}
\vskip -0.15in
\label{tbl:exp-confs}
\begin{center}
\begin{small}
\begin{sc}
\begin{tabular}{c|c}
& $\left(n,p,d,m,s\right)$\\ \hline
Low & $\left(100,1000,5,3,10\right)$ \\ \hline
Mod & $\left(200,10^4,5,3,10\right)$ \\ \hline
High & $\left(10^4,10^6,10,100,100\right)$
\end{tabular}
\end{sc}
\end{small}
\end{center}
\vskip -0.1in
\end{table}

\begin{figure*}[t]
\centering
\begin{adjustbox}{width=\textwidth}
\input{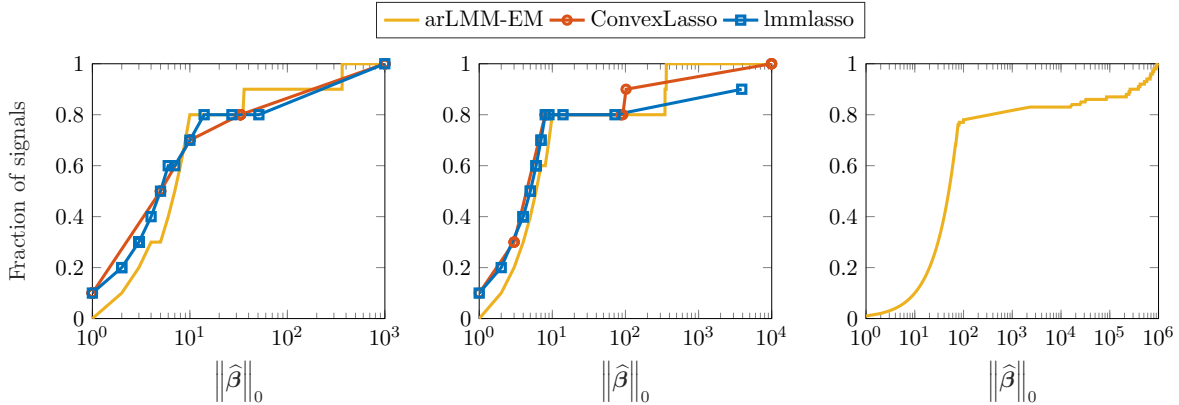}
\end{adjustbox}
\vskip -0.1in
\caption{Fraction of signals captured by $\widehat{\bm{\beta}}$. From left to right, the configurations are respectively {\sc Low}, {\sc Mod}, and {\sc High} in \cref{tbl:exp-confs}. It shows that {\tt arLMM-EM} performs competitively in variable selection.}
\label{fig:high}
\end{figure*} 

\cref{fig:high} reports variable selection results for {\tt arLMM-EM}, {\tt lmmlasso} \cite{Schelldorfer11}, and {\tt ConvexLasso}. All the settings in \cref{tbl:exp-confs} have sparse ground-truth $\bm{\beta}_{\text{True}}$. \cref{fig:high} shows the fraction of the signal (non-zeros in $\bm{\beta}_{\text{True}}$) recovered in the estimate $\widehat{\bm{\beta}}$. We varied the regularization parameters to obtain $\widehat{\bm{\beta}}$ with different levels of sparsity $\left\|\widehat{\bm{\beta}}\right\|_0$. The entries with the largest magnitude of $\widehat{\bm{\beta}}$ is considered the signal in these evaluations. As can be seen, {\tt arLMM-EM} delivers a competitive signal recovery ratio for $p=10^3,10^4$, and scales to considerably larger dimensions $n=10^4$ and $p=10^6$, which the other two methods cannot handle. This is also shown by the runtime (\cref{tbl:high-runtime}) on a Linux workstation with 2.40GHz Intel Xeon E5-2695 CPU and 256Gb memory. Once the variables are selected, a second run of the algorithm using only the selected variables gives the desired fixed-effect coefficients.

\begin{table}[!ht]
\caption{Runtime for the high-dimensional experiments.}
\vskip -0.15in
\label{tbl:high-runtime}
\begin{center}
\begin{small}
\begin{sc}
\begin{tabular}{c|c|c|c}
 & Low & Mod & High \\ \hline
lmmlasso & 75 s & 2.5 hr & / \\ \hline
ConvexLasso & 1 min & 40 min & / \\ \hline
arLMM-EM & $< 1$ s& $< 1$ s & 1 hr
\end{tabular}
\end{sc}
\end{small}
\end{center}
\vskip -0.1in
\end{table}

\subsection{Genome wide association studies}
LMMs have been used extensively for mapping traits in statistical genetics. The problem formulation is that of regressing a quantitative or binary trait onto a high-dimensional vector of 450,000 single nucleotide polymorphisms (SNPs), or locations of discrete genetic variation, for each subject in the study. The random effects are driven by population structure or
the pairwise similarity or relatedness between individuals.

\begin{table*}[htb!]
\caption{Comparing the prediction performance as well as the runtime of {\tt BSLMM} and {\tt arLMM-AVC} on the WTCCC dataset. $\text{Corr}\left(\widehat{\bm{\beta}}_{\text{BSLMM}},\widehat{\bm{\beta}}_{\text{arLMM-AVC}}\right)$ denotes the correlation between the fixed-effect coefficient estimates given by {\tt BSLMM} and {\tt arLMM-AVC}. }
\vskip -0.15in
\label{tbl:wtccc}
\begin{center}
\begin{small}
\begin{sc}
\begin{tabular}{cccccc}
\hline
\multirow{2}{*}{Disease} & 
\multicolumn{2}{c}{Time (min)} & \multicolumn{2}{c}{AUC} & \multirow{2}{*}{$\text{Corr}\left(\widehat{\bm{\beta}}_{\text{BSLMM}},\widehat{\bm{\beta}}_{\text{arLMM-AVC}}\right)$}\\
\cline{2-3} \cline{4-5}
 & BSLMM & arLMM-AVC & BSLMM & arLMM-AVC\\
\hline
BD   & 115.8 & 25.1 & 0.6520 & 0.6461 & 0.9898\\
CAD  & 161.0 & 26.1 & 0.5899 & 0.5937 & 0.9776\\
CD   & 110.3 & 25.4 & 0.6260 & 0.6328 & 0.9862\\
HT   & 120.6 & 19.4 & 0.5956 & 0.6010 & 0.9766\\
RA   & 147.4 & 19.9 & 0.6173 & 0.6206 & 0.9834\\
T1D  & 120.0 & 20.4 & 0.6846 & 0.6840 & 0.9939\\
T2D  & 155.3 & 18.9 & 0.6003 & 0.5993 & 0.9783\\
\hline
\end{tabular}
\end{sc}
\end{small}
\end{center}
\end{table*}

We compare our approximate estimator to the performance of a state-of-the-art estimator called {\tt BSLMM} (Bayesian sparse linear mixed model) \cite{Zhou13}. Specifically, we run BSLMM in its ridge-regression with mixed models setting, the fastest setting of the package for a fair comparison. In this setting  {\tt BSLMM} is computing the maximum a posteriori estimate of the regularized LMM. We compare performance on the Wellcome Trust Case Control Consortium  (WTCCC) dataset of 14,000 cases of 7 diseases - bipolar disorder (BD), coronary artery disease (CAD), Crohn's disease (CD), hypertension (HT), rheumatoid arthritis (RA), type 1 diabetes (T1D), and type 2 diabetes (T2D) - and 3,000 shared controls. This dataset characterizes over 450,000 single nucleotide polymorphisms (SNPs), or locations of discrete genetic variation, for each subject included in the study. Disease status is indicated as a binary response ($1$ for disease case, $-1$ for control). Each of the datasets had roughly equal numbers of cases and controls. 

For this experiment, we adopted the same random-effect covariance parameterization used to control for population structure $\theta \bm{X}\bm{X}^\top/p$ as {\tt BSLMM}, and used {\tt arLMM-AVC} with AVCs \eqref{eq:k-est} and \eqref{eq:sigma-est}. {\tt arLMM-AVC} and {\tt {\tt BSLMM}} were run under identical conditions on each of the seven approximately
5,000-subject $\times$ 450,000-SNP datasets. This was the same experimental setup used to validate {\tt BSLMM} in \cite{Zhou13}.

Observed runtimes for each of the seven datasets are reported in \cref{tbl:wtccc}. Correlation between the $\widehat{\bm{\beta}}$ reported by {\tt arLMM-AVC} and {\tt BSLMM} in all cases was very high, 0.977 or greater.

We also compared disease status prediction by splitting each dataset into a training set comprised of 80\% of subjects and a test set of the remaining 20\%, selected at random. {\tt arLMM-AVC } and {\tt BSLMM} each estimated $\widehat{\bm{\beta}}$ from the training set and attempted to predict disease status on the held-out set. We repeated this 20 times for each of the seven datasets and evaluated performance of prediction on the held-out set by area under the ROC curve (AUC). These results are also given in \cref{tbl:wtccc}. Predictive performance by {\tt arLMM-AVC} and {\tt BSLMM} was almost identical. Predicting disease status from genetic markers is hard and it is well known
that the effect sizes of genetic variants are individually small and that a great deal of variance in the response will also be driven by environmental factors.

\section{Conclusions}
\label{sec:concl}
State-of-the-art parameter inference in LMMs requires computational complexity which depends at least linearly on the number of covariates $p$ and generally relies on heuristics. In this paper, we presented scalable learning algorithms which have sublinear computational complexity in $p$ and provide theoretical guarantees for the accuracy of parameter estimation. Our approach combines novel approximate estimators that use a kernel matrix of the observations and the subsampled randomized Hadamard transform. Experiments on synthetic and real data corroborate the theory.

\section*{Code and Data}
The code of the {\tt arLMM} implementation as well as synthetic data generation is available on the Git repository: \url{https://github.com/ZilongTan/arLMM}. This study also makes use of data generated by the WTCCC. The WTCCC data can be obtained from \url{https://www.wtccc.org.uk}, where a full list of the investigators who contributed to the generation of the data are also listed. Funding for the WTCCC project was provided by the Wellcome Trust.

\section*{Acknowledgements}
Z.T. is supported by grant NSF CNS-1423128. S.M. would like to acknowledge the support of grants NSF IIS-1546331, NSF DMS-1418261, NSF IIS-1320357, NSF DMS-1045153, and NSF DMS-1613261. X.Z. would like to acknowledge the support of grants NIH R01HG009124 and NSF DMS-1712933.

\bibliography{ref}
\bibliographystyle{alpha}

\end{document}

%% file: def.tex
\usepackage[utf8]{inputenc}

\usepackage{bm}
\usepackage{amsmath}
\usepackage{amsfonts}
\usepackage{amssymb}
\usepackage{amsthm}
\usepackage{mathtools}
\usepackage{algorithm,algorithmic}
\usepackage{breqn}

\newtheorem{theorem}{Theorem}
\newtheorem{lemma}{Lemma}

\usepackage{threeparttable}
\usepackage{multirow}
\usepackage{fullpage}
\usepackage{adjustbox}
\usepackage{graphicx}
\usepackage{pgfplots}
\pgfplotsset{compat=newest}
\usepackage{subfig}

\DeclareMathOperator{\tr}{tr}

\DeclarePairedDelimiter\ceil{\lceil}{\rceil}

\newcommand{\vertiii}[1]{{\left\vert\kern-0.25ex\left\vert\kern-0.25ex\left\vert #1\right\vert\kern-0.25ex\right\vert\kern-0.25ex\right\vert}}

\definecolor{DarkGreen}{rgb}{0.1,0.5,0.1}
\definecolor{DarkRed}{rgb}{0.5,0.1,0.1}
\definecolor{DarkBlue}{rgb}{0.1,0.1,0.5}

\usepackage[colorlinks,bookmarks=true]{hyperref}
\hypersetup{
    citecolor=DarkBlue,  
    urlcolor=DarkGreen,   
    linkcolor=DarkRed
}

\usepackage[capitalize,nameinlink,noabbrev]{cleveref}
\crefname{section}{\S}{\S\S}
\Crefname{section}{\S}{\S\S}

%% file: corr.tex
\definecolor{mycolor1}{rgb}{0.00000,0.44700,0.74100}%
\definecolor{mycolor2}{rgb}{0.85000,0.32500,0.09800}%
\definecolor{mycolor3}{rgb}{0.92900,0.69400,0.12500}%
\begin{tikzpicture}

\begin{axis}[%
width=3.456in,
height=2.321in,
at={(-0.05in,0.182in)},
scale only axis,
xmin=0,
xmax=1,
ymin=0,
ymax=1,
ylabel style={font=\color{white!15!black}},
ylabel={Correlation},
axis line style={draw=none},
ticks=none,
axis x line*=bottom,
axis y line*=left,
legend style={at={(0.11,0.912)}, anchor=south west, legend columns=3, legend cell align=left, align=left, draw=white!15!black}
]
\end{axis}

\begin{axis}[%
width=1.353in,
height=0.885in,
at={(2.305in,0.285in)},
scale only axis,
bar shift auto,
log origin=infty,
xmin=-4.71428571428571,
xmax=54.7142857142857,
xtick={ 5, 25, 45},
xlabel style={font=\color{white!15!black}},
xlabel={p},
ymin=0.95,
ymax=1.02,
axis background/.style={fill=white},
legend style={at={(0.11,0.912)}, anchor=south west, legend columns=3, legend cell align=left, align=left, draw=white!15!black}
]
\addplot[ybar, bar width=4.571, fill=mycolor1, draw=black, area legend] table[row sep=crcr] {%
5	0.996654259752197\\
25	0.99603636334606\\
45	0.990748309478056\\
};
\addplot[ybar, bar width=4.571, fill=mycolor2, draw=black, area legend] table[row sep=crcr] {%
5	0.996483174816245\\
25	0.996003564832339\\
45	0.990857674093423\\
};
\end{axis}

\begin{axis}[%
width=1.353in,
height=0.882in,
at={(0.525in,0.288in)},
scale only axis,
bar shift auto,
log origin=infty,
xmin=-4.71428571428571,
xmax=54.7142857142857,
xtick={ 5, 25, 45},
xlabel style={font=\color{white!15!black}},
xlabel={p},
ymin=0.95,
ymax=1.02,
ylabel style={font=\color{white!15!black}},
ylabel={${\text{\it{} m = 10}}$},
axis background/.style={fill=white},
legend style={at={(0.11,0.912)}, anchor=south west, legend columns=3, legend cell align=left, align=left, draw=white!15!black}
]
\addplot[ybar, bar width=4.571, fill=mycolor1, draw=black, area legend] table[row sep=crcr] {%
5	0.999079172785464\\
25	0.999101049835936\\
45	0.999036576085262\\
};
\addplot[ybar, bar width=4.571, fill=mycolor2, draw=black, area legend] table[row sep=crcr] {%
5	0.999056972515798\\
25	0.999111936989927\\
45	0.999033676054541\\
};
\end{axis}

\begin{axis}[%
width=1.353in,
height=0.882in,
at={(2.305in,1.517in)},
scale only axis,
bar shift auto,
log origin=infty,
xmin=-4.77777777777778,
xmax=54.7777777777778,
xtick={ 5, 25, 45},
ymin=0.95,
ymax=1.02,
axis background/.style={fill=white},
legend style={at={(-1.8,1.2)}, anchor=south west, legend columns=3, legend cell align=left, align=left, draw=white!15!black}
]
\addplot[ybar, bar width=3.556, fill=mycolor1, draw=black, area legend] table[row sep=crcr] {%
5	0.991213516150068\\
25	0.990494440587006\\
45	0.997051568851009\\
};
\addlegendentry{arLMM-EM}

\addplot[ybar, bar width=3.556, fill=mycolor2, draw=black, area legend] table[row sep=crcr] {%
5	0.990915333371925\\
25	0.990552026128295\\
45	0.997049679396952\\
};
\addlegendentry{REML}

\addplot[ybar, bar width=3.556, fill=mycolor3, draw=black, area legend] table[row sep=crcr] {%
5	0.9895\\
25	0.9908\\
45	0.9976\\
};
\addlegendentry{arLMM-AVC}

\node[below, align=center]
at (rel axis cs:0.5,1) {${\text{\it{} d = 25}}$};
\end{axis}

\begin{axis}[%
width=1.353in,
height=0.882in,
at={(0.525in,1.517in)},
scale only axis,
bar shift auto,
log origin=infty,
xmin=-4.77777777777778,
xmax=54.7777777777778,
xtick={ 5, 25, 45},
ymin=0.95,
ymax=1.02,
ylabel style={font=\color{white!15!black}},
ylabel={${\text{\it{} m = 1}}$},
axis background/.style={fill=white},
legend style={at={(0.11,0.912)}, anchor=south west, legend columns=3, legend cell align=left, align=left, draw=white!15!black}
]
\addplot[ybar, bar width=3.556, fill=mycolor1, draw=black, area legend] table[row sep=crcr] {%
5	0.999507836830713\\
25	0.998695782062222\\
45	0.998431709166807\\
};
\addplot[ybar, bar width=3.556, fill=mycolor2, draw=black, area legend] table[row sep=crcr] {%
5	0.999519812297908\\
25	0.998694175071189\\
45	0.998435545581158\\
};
\addplot[ybar, bar width=3.556, fill=mycolor3, draw=black, area legend] table[row sep=crcr] {%
5	0.998468015136546\\
25	0.997939870054938\\
45	0.995456300290602\\
};
\node[below, align=center]
at (rel axis cs:0.5,1) {${\text{\it{} d = 5}}$};
\end{axis}
\end{tikzpicture}%

%% file: nll.tex
\definecolor{mycolor1}{rgb}{0.00000,0.44700,0.74100}%
\definecolor{mycolor2}{rgb}{0.85000,0.32500,0.09800}%
\begin{tikzpicture}

\begin{axis}[%
width=3.456in,
height=2.312in,
at={(-0.1in,0.181in)},
scale only axis,
xmin=0,
xmax=1,
ymin=0,
ymax=1,
ylabel style={font=\color{white!15!black}},
ylabel={NLL},
axis line style={draw=none},
ticks=none,
axis x line*=bottom,
axis y line*=left
]
\end{axis}

\begin{axis}[%
width=1.353in,
height=0.881in,
at={(2.305in,0.284in)},
scale only axis,
bar shift auto,
xmin=-4.71428571428571,
xmax=54.7142857142857,
xtick={ 5, 25, 45},
xlabel style={font=\color{white!15!black}},
xlabel={p},
ymin=0,
ymax=4000,
axis background/.style={fill=white}
]
\addplot[ybar, bar width=4.571, fill=mycolor1, draw=black, area legend] table[row sep=crcr] {%
5	2559.12524931025\\
25	2478.44389674538\\
45	2592.52077191467\\
};
\addplot[forget plot, color=white!15!black] table[row sep=crcr] {%
-4.71428571428571	0\\
54.7142857142857	0\\
};
\addplot[ybar, bar width=4.571, fill=mycolor2, draw=black, area legend] table[row sep=crcr] {%
5	2544.82200263335\\
25	2467.44411331326\\
45	2566.92946543283\\
};
\addplot[forget plot, color=white!15!black] table[row sep=crcr] {%
-4.71428571428571	0\\
54.7142857142857	0\\
};
\end{axis}

\begin{axis}[%
width=1.353in,
height=0.878in,
at={(0.525in,0.287in)},
scale only axis,
bar shift auto,
xmin=-4.71428571428571,
xmax=54.7142857142857,
xtick={ 5, 25, 45},
xlabel style={font=\color{white!15!black}},
xlabel={p},
ymin=0,
ymax=4000,
ylabel style={font=\color{white!15!black}},
ylabel={${\text{\it{} m = 10}}$},
axis background/.style={fill=white}
]
\addplot[ybar, bar width=4.571, fill=mycolor1, draw=black, area legend] table[row sep=crcr] {%
5	1699.34412548719\\
25	1677.55345021798\\
45	1755.87992515149\\
};
\addplot[forget plot, color=white!15!black] table[row sep=crcr] {%
-4.71428571428571	0\\
54.7142857142857	0\\
};
\addplot[ybar, bar width=4.571, fill=mycolor2, draw=black, area legend] table[row sep=crcr] {%
5	1697.51287781496\\
25	1676.07541046209\\
45	1753.01235273251\\
};
\addplot[forget plot, color=white!15!black] table[row sep=crcr] {%
-4.71428571428571	0\\
54.7142857142857	0\\
};
\end{axis}

\begin{axis}[%
width=1.353in,
height=0.878in,
at={(2.305in,1.511in)},
scale only axis,
bar shift auto,
xmin=-4.71428571428571,
xmax=54.7142857142857,
xtick={ 5, 25, 45},
ymin=0,
ymax=4000,
axis background/.style={fill=white}
]
\addplot[ybar, bar width=4.571, fill=mycolor1, draw=black, area legend] table[row sep=crcr] {%
5	2760.74898189307\\
25	2926.90889357527\\
45	2267.42676610884\\
};
\addplot[forget plot, color=white!15!black] table[row sep=crcr] {%
-4.71428571428571	0\\
54.7142857142857	0\\
};
\addplot[ybar, bar width=4.571, fill=mycolor2, draw=black, area legend] table[row sep=crcr] {%
5	2755.47627749073\\
25	2908.5404453831\\
45	2241.00818017105\\
};
\addplot[forget plot, color=white!15!black] table[row sep=crcr] {%
-4.71428571428571	0\\
54.7142857142857	0\\
};
\node[below, align=center]
at (rel axis cs:0.5,1) {${\text{\it{} d = 25}}$};
\end{axis}

\begin{axis}[%
width=1.353in,
height=0.878in,
at={(0.525in,1.511in)},
scale only axis,
bar shift auto,
xmin=-4.71428571428571,
xmax=54.7142857142857,
xtick={ 5, 25, 45},
ymin=0,
ymax=4000,
ylabel style={font=\color{white!15!black}},
ylabel={${\text{\it{} m = 1}}$},
axis background/.style={fill=white}
]
\addplot[ybar, bar width=4.571, fill=mycolor1, draw=black, area legend] table[row sep=crcr] {%
5	1663.8592930309\\
25	1780.8468788423\\
45	1901.34130054963\\
};
\addplot[forget plot, color=white!15!black] table[row sep=crcr] {%
-4.71428571428571	0\\
54.7142857142857	0\\
};
\addplot[ybar, bar width=4.571, fill=mycolor2, draw=black, area legend] table[row sep=crcr] {%
5	1662.39943541614\\
25	1779.79139808758\\
45	1900.10959375033\\
};
\addplot[forget plot, color=white!15!black] table[row sep=crcr] {%
-4.71428571428571	0\\
54.7142857142857	0\\
};
\node[below, align=center]
at (rel axis cs:0.5,1) {${\text{\it{} d = 5}}$};
\end{axis}
\end{tikzpicture}%

%% file: main.bbl
\newcommand{\etalchar}[1]{$^{#1}$}
\begin{thebibliography}{CRHD{\etalchar{+}}13}

\bibitem[BMBW15]{Bates15}
Douglas Bates, Martin M{\"a}chler, Ben Bolker, and Steve Walker.
\newblock Fitting linear mixed-effects models using {lme4}.
\newblock {\em Journal of Statistical Software}, 67(1):1--48, 2015.

\bibitem[CNZ04]{Cui04}
Hengjian Cui, Kai~W. Ng, and Lixing Zhu.
\newblock Estimation in mixed effects model with errors in variables.
\newblock {\em Journal of Multivariate Analysis}, 91(1):53--73, 2004.

\bibitem[CRHD{\etalchar{+}}13]{Chung13}
Yeojin Chung, Sophia Rabe-Hesketh, Vincent Dorie, Andrew Gelman, and Jingchen
  Liu.
\newblock A nondegenerate penalized likelihood estimator for variance
  parameters in multilevel models.
\newblock {\em Psychometrika}, 78(4):685--709, 2013.

\bibitem[Dem13]{Demidenko13}
Eugene Demidenko.
\newblock {\em Mixed Models: Theory and Applications with R}.
\newblock Wiley, 2nd edition, 2013.

\bibitem[DGME17]{Darnell17}
Gregory Darnell, Stoyan Georgiev, Sayan Mukherjee, and Barbara~E Engelhardt.
\newblock {Adaptive randomized dimension reduction on massive data}.
\newblock {\em Journal of Machine Learning Research (JMLR)}, 18(140):1--30,
  2017.

\bibitem[Har74]{Harville74}
David~A. Harville.
\newblock Bayesian inference for variance components using only error
  contrasts.
\newblock {\em Biometrika}, 61:383--385, 1974.

\bibitem[HY03]{Hall03}
Peter Hall and Qiwei Yao.
\newblock Inference in components of variance models with low replication.
\newblock {\em Annals of Statistics}, 31(2):414--441, 2003.

\bibitem[Jak15]{Jakubik15}
Jozef Jakub\'{i}k.
\newblock Convex method for variable selection in high-dimensional linear mixed
  models.
\newblock In {\em Proceedings of the 10th International Conference on
  Measurement}, pages 55--58, 2015.

\bibitem[JLR07]{Johnson07}
W.~Evan Johnson, Cheng~Li Li, and Ariel Rabinovic.
\newblock Adjusting batch effects in microarray expression data using empirical
  bayes methods.
\newblock {\em Biostatistics}, 8:118--127, Jan 2007.

\bibitem[KZW{\etalchar{+}}08]{Kang08}
Hyun~Min Kang, Noah~A Zaitlen, Claire~M Wade, Andrew Kirby, David Heckerman,
  Mark~J Daly, and Eleazar Eskin.
\newblock Efficient control of population structure in model organism
  association mapping.
\newblock {\em Genetics}, 178:1709--23, Mar 2008.

\bibitem[LB88]{Lindstrom88}
Mary~J. Lindstrom and Douglas~M. Bates.
\newblock {Newton-Raphson and EM algorithms for linear mixed-effects models for
  repeated-measures data}.
\newblock {\em Journal of the American Statistical Association}, 83:1014--1022,
  1988.

\bibitem[LDFU13]{Lu13}
Yichao Lu, Paramveer Dhillon, Dean~P Foster, and Lyle Ungar.
\newblock {Faster Ridge Regression via the Subsampled Randomized Hadamard
  Transform}.
\newblock In {\em Advances in Neural Information Processing Systems (NIPS)},
  pages 369--377. 2013.

\bibitem[LLL{\etalchar{+}}11]{Lippert11}
C.~Lippert, J.~Listgarten, Y.~Liu, CM. Kadie, RI. Davidson, and D.~Heckerman.
\newblock Fast linear mixed models for genome-wide association studies.
\newblock {\em Nature Methods}, 8(10):833–835, October 2011.

\bibitem[LLS87]{Laird87}
Nan Laird, Nicholas Lange, and Daniel Stram.
\newblock {Maximum likelihood computations with repeated measures: Application
  of the EM algorithm}.
\newblock {\em Journal of the American Statistical Association}, 82:97--105,
  1987.

\bibitem[LW82]{Laird82}
N.~M. Laird and J.~H. Ware.
\newblock Random-effects models for longitudinal data.
\newblock {\em Biometrics}, 38(4):963--974, December 1982.

\bibitem[Per17]{Perry17}
Patrick~O. Perry.
\newblock Fast moment-based estimation for hierarchical models.
\newblock {\em Journal of the Royal Statistical Society: Series B (Statistical
  Methodology)}, 79(1):267--291, 2017.

\bibitem[PT71]{Patterson71}
H.~D. Patterson and R.~Thompson.
\newblock Recovery of inter-block information when block sizes are unequal.
\newblock {\em Biometrika}, 58:545--554, 1971.

\bibitem[SBDG11]{Schelldorfer11}
Jürg Schelldorfer, Peter Bühlmann, and Sara~Van De~Geer.
\newblock Estimation for high-dimensional linear mixed-effects models using
  $\ell_1$-penalization.
\newblock {\em Scandinavian Journal of Statistics}, 38(2):197--214, 2011.

\bibitem[SGV98]{Saunders98}
Craig Saunders, Alexander Gammerman, and Volodya Vovk.
\newblock {Ridge Regression Learning Algorithm in Dual Variables}.
\newblock In {\em Proceedings of the Fifteenth International Conference on
  Machine Learning (ICML)}, pages 515--521, 1998.

\bibitem[SMB14]{Schelldorfer14}
J{\"u}rg Schelldorfer, Lukas Meier, and Peter B{\"u}hlmann.
\newblock {GLMMLasso: An Algorithm for High-Dimensional Generalized Linear
  Mixed Models Using $\ell_1 $-Penalization}.
\newblock {\em Journal of Computational and Graphical Statistics},
  23(2):460--477, 2014.

\bibitem[SS90]{Stewart90}
G.W. Stewart and Jiguang Sun.
\newblock {\em Matrix Perturbation Theory}.
\newblock Computer science and scientific computing. Academic Press, 1990.

\bibitem[Tro11]{Tropp11}
Joel~A. Tropp.
\newblock {Improved Analysis of the subsampled Randomized Hadamard Transform}.
\newblock {\em Advances in Adaptive Data Analysis}, 3(1-2):115--126, 2011.

\bibitem[{Wel}07]{WTCCC}
{Wellcome Trust Case Control Consortium}.
\newblock {Genome-wide association study of 14,000 cases of seven common
  diseases and 3,000 shared controls}.
\newblock {\em Nature}, 447(7145):661--678, June 2007.

\bibitem[ZA08]{Zhang08}
Liang Zhang and Deepak Agarwal.
\newblock {Fast Computation of Posterior Mode in Multi-Level Hierarchical
  Models}.
\newblock In {\em Advances in Neural Information Processing Systems (NIPS)},
  pages 1913--1920, 2008.

\bibitem[ZCS13]{Zhou13}
Xiang Zhou, Peter Carbonetto, and Matthew Stephens.
\newblock Polygenic modeling with bayesian sparse linear mixed models.
\newblock {\em PLOS Genetics}, 9(2):1--14, 02 2013.

\bibitem[ZDJ11]{Zhang11}
Zhihua Zhang, Guang Dai, and Michael~I. Jordan.
\newblock {Bayesian Generalized Kernel Mixed Models}.
\newblock {\em Journal of Machine Learning Research (JMLR)}, 12:111--139,
  February 2011.

\bibitem[Zho17]{Zhou17}
Xiang Zhou.
\newblock A unified framework for variance component estimation with summary
  statistics in genome-wide association studies.
\newblock {\em The Annals of Applied Statistics}, 11(4):2027--2051, 2017.

\bibitem[ZS14]{Zhou14}
Xiang Zhou and Matthew Stephens.
\newblock Efficient multivariate linear mixed model algorithms for genome-wide
  association studies.
\newblock {\em Nature Methods}, 2014.

\bibitem[ZZM{\etalchar{+}}16]{Zhang16}
Xianxing Zhang, Yitong Zhou, Yiming Ma, Bee{-}Chung Chen, Liang Zhang, and
  Deepak Agarwal.
\newblock {GLMix: Generalized Linear Mixed Models For Large-Scale Response
  Prediction}.
\newblock In {\em International Conference on Knowledge Discovery and Data
  Mining (KDD)}, pages 363--372, 2016.

\end{thebibliography}
